\documentclass{article} 
\usepackage{nips12submit_e,times}
\usepackage{amssymb}
\usepackage{amsmath}
\usepackage{amsthm}
\usepackage{bm}
\usepackage{url}
\usepackage[ruled,vlined]{algorithm2e}
\usepackage{graphicx}

\newif\ifarxiv
\newif\ifdraft
\arxivtrue
\newcommand{\arxivVer}[1]{\ifarxiv #1 \fi}
\newcommand{\wsVer}[1]{\ifarxiv \else #1 \fi}
\newcommand{\switchVer}[2]{\ifarxiv #2 \else #1 \fi}
\newcommand{\draftfootnote}[1]{\ifdraft .\footnote{**** #1} \fi}

\newtheorem{theorem}{Theorem}
\newtheorem{lemma}{Lemma}

\newtheorem{corollary}{Corollary}
\newtheorem{definition}{Definition}

\renewcommand{\vec}[1]{\bm{#1}}
\renewcommand{\v}[1]{\vec{#1}}
\renewcommand{\t}[1]{\tilde{#1}}
\newcommand{\bb}[1]{\mathbb{#1}}
\newcommand{\eqn}[1]{\begin{align}#1\end{align}}
\renewcommand{\L}{\mathcal{L}}
\newcommand{\LL}{\Lambda}
\newcommand{\V}{\mathcal{V}}
\newcommand{\T}{\mathcal{T}}
\newcommand{\U}{\mathcal{U}}

\newcommand{\mat}[1]{\left( \begin{matrix}#1\end{matrix} \right)}
\newcommand{\diag}[1]{\mbox{diag}\mat{#1}}
\newcommand{\ai}{\alpha_i}
\newcommand{\E}{\bb{E}}
\newcommand{\eps}{\varepsilon}

\title{Submodularity  
in Batch Active Learning and Survey Problems on Gaussian Random Fields}

\arxivVer{
\author{
Yifei Ma\\
Machine Learning Department,\\
Carnegie Mellon University\\
\texttt{yifeim@cs.cmu.edu} \\
\And
Roman Garnett \\
Robotics Institute, \\
Carnegie Mellon University \\
\texttt{rgarnett@cs.cmu.edu} \\
\And
Jeff Schneider \\
Robotics Institute, \\
Carnegie Mellon University \\
\texttt{schneide@cs.cmu.edu} \\
}
}
\nipsfinalcopy 

\begin{document}

\maketitle

\wsVer{\vspace{-5em}}

\begin{abstract}

Many real-world datasets can be represented in the form of a graph whose edge weights designate similarities between instances. A discrete Gaussian random field (GRF) model is a finite-dimensional Gaussian process (GP) whose prior covariance is the inverse of a graph Laplacian. 
Minimizing the trace of the prediction covariance $\Sigma$ (V-optimality) on GRFs has proven successful in batch active learning classification problems with budget constraints.
However, its worst-case bound has been missing. We show that the V-optimality on GRFs as a function of the batch query set is submodular and hence its greedy selection algorithm guarantees an $(1-1/e)$ approximation ratio. Moreover, GRF models have the absence-of-suppressor (AofS) condition. 
For active survey problems, we propose a similar survey criterion which minimizes $\v{1}^T\Sigma\v{1}$. In practice, V-optimality criterion performs better than GPs with mutual information gain criteria   and allows nonuniform costs for different nodes. 

\end{abstract}

\section{Introduction}
In many real-world applications, such as author classification based on coauthorship graphs, 
one or more output variables need to be predicted from a subset of queryable inputs, constrained by a budget.
In batch active learning applications, an algorithm refines its prediction by generating a list of queries for domain experts to answer [2, 5, 8]. In both cases, we consider the situation where 
the similarities between all instances, both labeled and unlabeled, are known a-priori. We formulate these similarities by a graph $G=(V,E)$ with edge weights $W$.
The goal is to optimize the subset of nodes to query within the budget so that the risk in prediction can be minimized. One common risk is the predictive variance, measured by the trace of the covariance matrix of multivariate outputs. Minimizing this risk is  known as the V-optimality criterion. 

Commonly used models  for these subset selection or batch active learning problems are discrete Gaussian random fields (GRF) [2, 5], finite-dimensional Gaussian processes (GP) [1], and linear regression with prior knowledge of covariances [3, 4].  GRFs formulate the input-output correspondence by the conditional distribution of a (maybe improper) gaussian prior whose inverse covariance is set to be the graph Laplacian, sometimes with diagonal regularization. Finite-dimensional GPs define the prior as $\mathcal{N}(0,W)$, where $W$ is an arbitrary covariance matrix. Finally, linear regression with prior knowledge of covariance is essentially a finite-dimensional GP with linear covariance. 

GPs have been used as a  base model for both subset selection and active learning [1]. One minor issue is that they require $W$ to be positive-semidefinite. However a major issue is that they do not have a provable lower bound for optimality [4]. Instead, [1] used an alternative mutual information gain (MIG) criterion for  selecting nodes for query. The MIG-criterion is naturally a normalized, monotone, and submodular function. As a result, a greedy algorithm gaurantees an $(1-1/e)$ approximation ratio. However, there is not classification-related risk function associated and the log determinates of covariance submatrices are sensitive to small eigenvalues, which can be a problem.\draftfootnote{**** I just realized that they changed from information gain to mutual information gain and our old explanation does not hold.  ????}

Another direction with GP models is to constrain the prior kernel matrix. [4] constrained the prior covariance matrix such that its diagonal entries are 1s and off-diagonal entries some very small values. 
However, these models can be approximated by regularized GRF models  in that $(I + \eps W)^{-1} =  I - \eps W + \eps^2W^2 - \cdots\approx  I-\eps W$,  when $\lim\eps^sW^s=0$, with small $\eps>0$.
[4] also proposed an absence-of-suppressors (AofS) condition that is sufficient for submodularity. However, it is generally hard to verify if a discrete GP meets the AofS condition, whereas we will show that every GRF is AofS.

Finally, [5] demonstrates semi-supervised and active learning using GRFs. Their motivation is that unlabeled nodes can reasonably influence the prediction by their edge weights with other nodes, because these weights can encode information such as sample density (e.g. using a radial basis function kernel to calculate the weight matrix). Later research [1] used spectral methods to boost the computation speed for subset selection in batch active learning. However, they only solved the subset selection case where every node query has a unit cost. Moreover, in both works, the optimization lacks worst-case guarantees. 

In this paper, we properly define a (regularized) discrete GRF model and prove an $(1-1/e)$ approximation ratio lower bound with the V-optimality criterion under a limited budget for a greedy subset selection algorithm. We also extend this bound for the scenario where different nodes have different costs. GRF models are a special type of AofS GP models. Conversely, any GP model whose conditional covariance matrices are always nonnegative is a GRF model and is AofS. From real-world experiments we show that GRF models using the V-optimality criterion present advantages over GP models with the MIG criterion and random selection.

\section{Gaussian Random Fields and Subset Selection Problems}

\subsection{The Gaussian Random Field (GRF) model}

Suppose the dataset can be represented in the form of a connected undirected graph $G=(V,E)$ where each node has an (either known or unknown) label and each edge $e_{ij}$ has a fixed nonnegative weight $w_{ij}(=w_{ji})$ that reflects the proximity, similarity, etc between nodes $v_i$ and $v_j$. Define the graph Laplacian of $G$ to be $L_0=\diag{W}-W$ and the regularized graph Laplacian to be $L_\sigma=L_0+\diag{\sigma_1^{-2},...,\sigma_N^{-2}}$ with $\sigma_i>0,\forall i=1,...,N$. We use $L$ to generalize both.

The discrete Gaussian Random Field (GRF) is a joint \emph{continuous} distribution on both labeled and unlabeled nodes, containing one tunable ``heat" parameter $\beta>0$, as
\eqn{
\bb{P}(\vec{y})\propto \exp\left(-\frac{\beta}{2}\v{y}^TL\v{y}\right)
= \begin{cases} 
\exp\left(-\frac{\beta}{2}\sum_{i,j}w_{ij}(y_i-y_j)^2 \right) & \mbox{(unregularized)}
\\
\exp\left(-\frac{\beta}{2}\sum_{i,j}w_{ij}(y_i-y_j)^2 + \sum_i \frac{1}{\sigma_i^2}y_i^2 \right) & \mbox{(regularized)}.
\end{cases}\label{eqn:theGRFmodel}
}
Assuming labels $\vec{y}_\mathcal{L} = \{y_{l_1},...,y_{l_{|\L|}}\}$ are tagged as $\vec{t}_\mathcal{L} \in [0,1]^{|\L|}$, a Gaussian Harmonic predictor predicts all unlabeled \emph{continuous} nodes  $\vec{y}_\mathcal{U}=\{y_{u_{1}},...,y_{u_{|\U|}}\}$ by \emph{factoring out} known variables [5],
\eqn{\bb{P}(\vec{y}_\U|\vec{y}_\mathcal{L} = \vec{t}_\mathcal{L}) \sim \mathcal{N}(\vec{f}_\mathcal{U},  \beta L_{\U}^{-1}) = \mathcal{N}(\vec{f}_\mathcal{U},  \beta L_{(\V-\L)}^{-1}) , \label{eqn:GaussianHarmonicPrediction}}
where $L_\U$ is the submatrix consisting of the unlabeled row and column indices in $L$, for example the lower right block of $L = \left( \begin{array}{cc} L_{ll} & L_{lu} \\ L_{ul} & L_{uu} \end{array}\right)$ and $\vec{f}_\mathcal{U} = (-L_{\U}^{-1}L_{ul}\vec{t}_\mathcal{L})$. By convention, $L_{(\V-\L)}^{-1}$ means the inverse of the submatrix. We use $L_{(\V-\L)}$ and $L_\U$ interchangeably because $\L$ and $\U$ partition the set of all nodes $\V$.

In some problems, a test set $\T\subset\U$ is specified. Define $T$ to be a $|\T|\times|\U|$ matrix such that $t_{ij}=\delta(v_{t_i},v_{u_j})$, i.e. $T\v{y}_\U = \v{y}_\T$. Otherwise, a default value of $T$ is the identity matrix of size $|\U|$. By \emph{marginalizing out} node variables  in $\U\backslash\T$ from (\ref{eqn:GaussianHarmonicPrediction}), we have
$
\bb{P}(\vec{y}_\T|\vec{t}_\mathcal{L}) \sim \mathcal{N}(T\vec{f}_\mathcal{U},  \beta TL_{(\V-\L)}^{-1}T^T).
$

Notice that GRFs differ from general GPs in that the predictive mean $\vec{f}_\mathcal{U} \in [0,1]^{N-|\L|}$ (Corollary 1). Unlike GPs, GRFs do not ``squeeze"  regression responses to $[0,1]$ to get probability predictions.

\subsection{Risk Minimization for Classification}

Since in GRFs regression responses are taken directly as probability predictions, it is computationally and analytically more convenient to apply the regression loss and risk directly in the GRF as in [2]. Assume the L2 loss to be our classification loss,
$
L_c(\v{y}_\T,\v{f}_\T) = \sum_{v_{t_i}\in\T}(y_{t_i}-f_{t_i})^2, $
and a risk function whose input variable is the subset $\L$ as
\eqn{R_c(\mathcal{L}) \textstyle
=\E^{\v{y}_\L\v{y}_\T}L_c(\v{y}_\T, \v{f}_\T | \v{y}_\L) 
= \E \E\left[\sum_i\big(\v{y}_{t_i} + {(TL_{\U}^{-1}L_{ul}\vec{y}_\mathcal{L})}_i\big)^2\middle| \v{y}_\L \right] 
= tr(TL_{\U}^{-1}T^T)  \label{eqn:Risk}
}
\subsection{The Subset Selection  Problem (the Active Learning for Classification Problem)}
Assume every vertex on the graph has a cost, (a unit cost if not specified), the major objective in this paper is to choose a subset of nodes $\mathcal{L}=\{v_{l_1},...,v_{l_{|\L|}}\}$ to query for labels, constrained by a given budget $C$, such that the risk is minimized. Formally,
\eqn{
\textstyle\arg  \min_{\mathcal{L}}  &\textstyle \quad R(\L)=R_c(\L)=tr(TL_{(\V-\L)}^{-1}T^T)  \nonumber\\
s.t. &\textstyle \quad \sum_{v\in \mathcal{L}}  c_v \leq C \label{eqn:problem}
}
Though not explicitly denoted, the specific matrix $T=(\delta(v_{t_i},v_{u_j}))_{i=1,j=1}^{|\T|, |\U|}$ depends on $\U=\V-\L$.

\section{Submodularity, Suppressor-free, and Bounds for Greedy Method}
In \S~3, we assume that $L$ is nonsingular. This could be achieved by either deleting a node (a row and a column) from the original undirected connected graph Laplacian, i.e. assuming that the dataset always contains a fixed label, or by using the regularized $L_\sigma$. In these cases, $L$ satisfies the following.  
\eqn{ 
\bullet \quad&L \mbox{ has proper signs, i.e. }l_{ij}  \geq 0  \mbox{ if } i=j \mbox{ and } l_{ij} \leq 0 \mbox{ if } i\neq j; \label{ppt1}\\
 \bullet \quad& {\textstyle L \mbox{ is undirected and connected, i.e. } l_{ij}  = l_{ji} \forall i,j\mbox{ and }\sum_{j\neq i}(-l_{ij})>0 \quad\forall i; }\label{ppt2}\\ 
\bullet \quad& {\textstyle \mbox{Node degree no less than number of edges, i.e. }\sum_j l_{ij} = \sum_j l_{ji} \geq 0 \quad\forall i=1,...,N; } \label{ppt3}\\
\bullet \quad&L {\textstyle \mbox{ is nonsingular and therefore positive definite, i.e. }\exists i \mbox{ s.t. } \sum_j l_{ij} = \sum_j l_{ji} > 0 . } \label{ppt4}
}
Conversely, our results hold if a finite-dimensional GP has covariance $=(L^{-1})$ and $L$ satisfies (\ref{ppt1}-\ref{ppt4}).
\subsection{Major results}
$\bullet$ \emph{Submodularity.} Under conditions (\ref{ppt1}-\ref{ppt4}), 
the risk reduction function $R_\Delta(\mathcal{L}) := R(\emptyset) - R(\L)$ is normalized, monotone, and submodular, i.e., 
\eqn{
R_\Delta(\emptyset) &= 0 \label{eqn:submod1} \\
R_\Delta(\mathcal{L}_1\cup\mathcal{L}_2) &\geq R_\Delta(\mathcal{L}_1) \label{eqn:submod2}\\
R_\Delta(\mathcal{L}_1\cup\{v\}) - R_\Delta(\mathcal{L}_1)
&\geq
R_\Delta(\mathcal{L}_1\cup\mathcal{L}_2\cup\{v\}) - R_\Delta(\mathcal{L}_1\cup\mathcal{L}_2) \label{eqn:submod3} \\
& \quad \forall \quad \mathcal{L}_1, \mathcal{L}_2, v \nonumber
}
$\bullet$ \emph{Greedy Algorithm and near-optimal bounds.}
If (\ref{eqn:submod1}-\ref{eqn:submod3}) is satified, the optimization problem (\ref{eqn:problem}) is NP-hard and the greedy selection algorithm (Algo~\ref{algo:greedy}) produces a query set $\L_g$ that gaurantee an $(1-1/e)$ optimality bound [6],
\eqn{
R_\Delta(\L_g) \geq (1-\frac{1}{e})\cdot R_\Delta(\L_*), \label{eqn:1-1/e}
}
where $\L_*$ is the global (NP) optimizer under the constraint $\sum_{v\in \mathcal{L_*}}  c_v \leq \sum_{v\in {\L_g}}c_v$. 

\begin{algorithm}
 \caption{{Greedy subset selection. Fast realization of * in [2], [5], and \switchVer{our longer version [9].}{\S~4.}}\label{algo:greedy}}
 \normalsize
 \KwIn{Node costs $c_v$, budget $C$, queryable pool $\mathcal{P}$, objective function $R(\L)$.}%
 \KwOut{A subset $\L\subset\mathcal{P}$ by greedy selection.}
 Define $\L\leftarrow\emptyset$, $R^{old} \leftarrow R(\emptyset)$.\\ %
 \While{available pool $\mathcal{P'} = \{v'\in \mathcal{P}-\L: c_{v'} + \sum_{v\in\L} c_v\leq C\}$ is not empty}{
\nlset{*\hspace{-2em}} Find $\displaystyle v'_* \leftarrow \arg\min_{v'\in\mathcal{P'}} \frac{R(\L\cup\{v'\})-R^{old} }{c_{v'}}$. \\
  Update $\L\leftarrow\L\cup\{v'_*\}, R^{old}\leftarrow R(\L).$ %
  }
\end{algorithm}

$\bullet$ \emph{Relationship with suppressor-free models.}
An absence-of suppressor (AofS) condition in regression models gaurantees submodularity. With our notation,\footnote
{$|Corr(Z, Res(X_i, S)/Res(X_j , S))| \leq |\rho(Z, Res(X_i, S))|$ in the original paper.} 
this condition is $|Corr(y_i, y_j | \L_1\cup \L_2)| \leq |Corr(y_i, y_j | \L_1)| ,\quad \forall v_i,v_j,\L_1,\L_2$. 
An example of suppressor variable is some node $v_k\in\L_2-\L_1$ such that $y_i+y_j=y_k$. Such variable is counter-tuitive in prediction models because knowing $y_k$ suppresses an unmodeled correlation between the predictors. We show that the GRF model is a perfect example for AofS condition. %

\subsection{Proofs}

\begin{lemma}
For any $L$ satisfying (\ref{ppt1}-\ref{ppt4}), the inverse of $L$ is nonnegative, i.e. $L^{-1}\geq 0$ (entry-wise).\footnote{In the following, for any vector or matrix $A$, $A\geq0$ always stands for $A$ being (entry-wise) nonnegative.}
\label{4}
\end{lemma}

\begin{proof}
Define $D = \diag{L}$ and $W = D-L$, we have
$L = D-W = D(I-D^{-1} W). \label{Ldecomp}$

According to (\ref{ppt1}), entry-wise $D \geq 0$, $W\geq0$ and $D^{-1}W\geq0$.  
Furthermore, by (\ref{ppt3}), 
\eqn{0\leq D^{-1}W 
&= \Big(\frac{w_{ij}}{d_{ii}}\Big)_{i,j=1}^N\leq \Big(\frac{w_{ij}}{\sum_{k}w_{ik}}\Big)_{i,j=1}^N, \\
\|D^{-1} W\|_\infty :=\sup_{\v{x}\neq0} \frac{\max_i |(D^{-1} W \v{x})_i|}{\max_i |x_i|} 
&=\max_i \sum_j|(D^{-1}W)_{ij}| \leq \max_i \sum_j\frac{w_{ij}}{\sum_kw_{ik}} \leq 1.
}
Thus, any eigenvalue $\lambda_k$ and its corresponding eigenvector $\v{v}_k$ of $D^{-1}W$ needs to satisfy $
|\lambda_k|\|\v{v}_k\|_\infty = \|\lambda_k\v{v}_k\|_\infty = \|D^{-1}W\v{v}_k\|_\infty \leq \|\v{v}_k\|_\infty,
 \quad$ i.e. $|\lambda_k|\leq1\quad\forall k=1,...,N$.

Moreover, (\ref{ppt4}) the invertibility of $L$  implies that $(I-D^{-1}W)$ is invertible, i.e. having no 0 eigenvalue. Hence, $|\lambda_k|<1, \forall k=1,...,N$ and $\lim_{n\to\infty} (D^{-1}W)^n = 0$. The latter yeilds the following,
\eqn{
L^{-1} = (1-D^{-1}W)^{-1}D^{-1}=[I+D^{-1}W +(D^{-1}W)^2+\cdots]D^{-1}. \label{Tayler}
}
Since every term in the right hand side of (\ref{Tayler}) is nonnegative, $L^{-1}$ should also be nonnegative.
\end{proof}

\begin{corollary}
GRF prediction functor $L_\U^{-1}L_{ul}$ maps  $\v{y}_\L\in[0,1]^{|\L|}$ to $\v{f}_\U = -L_\U^{-1}L_{ul}\v{y}_\L\in[0,1]^{|\U|}$.
\end{corollary}

\begin{proof}
Since $L_\U\geq0$ and $-L_{ul}\geq0$, we have 
$
\v{y}_\L\geq0 \Rightarrow L_\U^{-1}(-L_{ul})\v{y}_\L\geq0$ and $\v{y}_\L\geq\v{y}'_\L \Rightarrow L_\U^{-1}(-L_{ul})\v{y}_\L\geq L_\U^{-1}(-L_{ul})\v{y}'_\L.
$
On the other hand, $\mat{L_\U,L_{ul}}\cdot\v{1}\geq0$ and $L_\U^{-1}\geq0$ imply $\mat{I,L_\U^{-1}L_{ul}}\cdot\v{1}\geq0$, i.e. $\v{1}+L_\U^{-1}L_{ul}\v{1}\geq0 $. Hence, $ \v{1}\geq -L_\U^{-1}L_{ul}\v{1} \geq -L_\U^{-1}L_{ul}\v{y}_\L$.
\end{proof}

\begin{lemma}
Suppose $L=\mat{L_{11}&L_{12}\\L_{21}&L_{22}}$ satisfies (\ref{ppt1}-\ref{ppt4}), then $L^{-1} - \mat{L_{11}^{-1} &0\\0&0}$ is positive-semidefinite and nonnegative. 
\end{lemma}
\begin{proof}
By block matrix inversion theorem, 
\eqn{L^{-1} - \mat{L_{11}&0\\0&0} = \mat{-L_{11}^{-1}L_{12} \\ I}(L_{22}-L_{21}L_{11}^{-1}L_{12})^{-1}\mat{-L_{21}L_{11}^{-1} & I}\label{BIT}}
By assumption (\ref{ppt4}), $L^{-1}$ is positive-definite, so is its lower right corner $(L_{22}-L_{21}L_{11}^{-1}L_{12})^{-1}$. Thus, $L^{-1} - \mat{L_{11}&0\\0&0}$ is positive-semidefinite.

By Lemma 1, $L^{-1}\geq0$ and this implies that its lower right $(L_{22}-L_{21}L_{11}^{-1}L_{12})^{-1}\geq0$. The submatrix $L_{11}$ also satisfies (\ref{ppt1}-\ref{ppt4}) and by Lemma 1, $L_{11}^{-1}\geq0$. By sign rule (\ref{ppt1}), $(-L_{12})=(-L_{21})^T\geq0$. Now that every term on  the right side of (\ref{BIT}) is nonnegative, the left side also has to be.
\end{proof}

\begin{lemma}[Monotonicity]
For  function $R_\Delta(\L)$ defined in \S~2.3, $R_\Delta(\L_1\cup\L_2)\geq R_\Delta(\L_1)$, $\forall \L_1,\L_2$.
\end{lemma}
\begin{proof}
Direct application of Lemma 2.
\end{proof}


\begin{lemma}[Submodularity]
For  function $R_\Delta(\L)$ defined in \S~2.3, $R_\Delta(\L_1\cup\{v\}) - R_\Delta(\L_1) \geq R_\Delta(\L_1\cup\L_2\cup\{v\})-R_\Delta(\L_1\cup\L_2) $, $\forall \L_1,\L_2, v$.
\end{lemma}

\begin{proof}
We may assume that $\L_1,\L_2,\mbox{ and }\{v\}$ are disjoint. Without loss of generality, suppose 
\eqn{L_{(\V-\L_1)} =& \left(\begin{array}{cc|c}L_{(\V-\L_1\cup\L_2\cup\{v\})} & L_{(\V-\L_1\cup\L_2\cup\{v\}),\L_2} & L_{(\V-\L_1\cup\L_2\cup\{v\}),\{v\}}\\
L_{\L_2 ,(\V-\L_1\cup\L_2\cup\{v\})} & L_{\L_2} & L_{\L_2,\{v\}}\\\hline
L_{\{v\},(\V-\L_1\cup\L_2\cup\{v\})} & L_{\{v\},\L_2} & L_{\{v\}}
\end{array}\right) \label{partition1}
}\eqn{
:=&\left(\begin{array}{ll|l}
\tilde{A} & Y & \tilde{b} \\ Y^T & Z & e \\\hline \tilde{b}^T & e^T & c
\end{array}\right) := \left(\begin{array}{l|l}
A & b \\\hline b^T & c 
\end{array}\right) \label{partition2}
\\ \hspace{-5em} \mbox{and }
L_{(\V-\L_1\cup\L_2)} =& \left(\begin{array}{c|c}L_{(\V-\L_1\cup\L_2\cup\{v\})}  & L_{(\V-\L_1\cup\L_2\cup\{v\}),\{v\}}\\\hline
L_{\{v\},(\V-\L_1\cup\L_2\cup\{v\})} & L_{\{v\}}
\end{array}\right)
= \left(\begin{array}{l|l}
\tilde{A} & \tilde{b} \\\hline \tilde{b}^T & c 
\end{array}\right).  \label{partition3}
}

Apply block matrix inversion theorem, when a test set is not specified, i.e. $T=I$ of size $|\U|$, 
\eqn{
&R_{\Delta}(\L_1\cup\{v\}) - R_{\Delta}(\L_1)
=
 R(\L_1) - R(\L_1\cup\{v\}) 
=  tr\left(\mat{A &b\\b^T & c}^{-1} - \mat{A^{-1} & 0\\0&0}\right) \nonumber\\
&=  tr\left(\mat{-A^{-1}b\\1}\frac{1}{c-b^TA^{-1}b}\mat{-b^TA^{-1} , 1}\right)
=  \frac{1}{c-b^TA^{-1}b}\mat{-b^TA^{-1} , 1}\mat{-A^{-1}b\\1} \label{submod:c1}
}
Similarly, $\displaystyle
R_{\Delta}(\L_1\cup\L_2\cup\{v\}) - R_{\Delta}(\L_1\cup\L_2)
=  \frac{(-\tilde{b}^T\tilde{A}^{-1})(\tilde{A}^{-1}(-\tilde{b}))+1}{c-(-\tilde{b})^T\tilde{A}^{-1}(-\tilde{b})}.
$

Notice that by sign rule (\ref{ppt1}), $-b\geq\mat{-\tilde{b}\\0}\geq0$ and by Lemma 2, $A^{-1}\geq \mat{\tilde{A}^{-1} & 0 \\ 0 & 0}\geq0$. Thus, $(-b^T)A^{-1}(-b)\geq\mat{-\tilde{b}^T,0}\mat{\tilde{A}^{-1} & 0 \\ 0 & 0}\mat{-\tilde{b}\\0}=(-\tilde{b}^T)\tilde{A}^{-1}(-\tilde{b})\geq0$ and
$A^{-1}(-b)\geq\mat{\tilde{A}^{-1} & 0 \\ 0 & 0}\mat{-\tilde{b}\\0}=\tilde{A}^{-1}(-\tilde{b})\geq0$. 
The proof when a test set $\T$ is specified is fundamentally similar because the indicator matrix $T$ is always applied to nonnegative vectors or matrices.
\end{proof}

\begin{theorem}[(1-1/e) Bound]
The  function $R_\Delta(\L)$ defined in \S~2.3  is normalized (by definition), monotone (by lemma 3), and submodular (by lemma 4). Therefore, (\ref{eqn:1-1/e}) can be established. \qed
\end{theorem}

\begin{definition}
Since the conditional covariance of a GRF model is $L_{\L}^{-1}$, we can properly define the corresponding conditional correlation to be \eqn{Corr(\v{y}_\U|\L) = \left(\mbox{diag}(L_{(\V-\L)}^{-1})^{-\frac{1}{2}}\right)L_{(\V-\L)}^{-1}\left(\mbox{diag}(L_{(\V-\L)}^{-1})^{-\frac{1}{2}}\right)}
\end{definition}

\begin{theorem}[AofS]
$Corr(y_i,y_j|\L_1)\geq Corr(y_i,y_j| \L_1\cup\L_2), \forall \L_1, \L_2 , \forall v_i,v_j \not\in \L_1\cup\L_2.$
\end{theorem}

\begin{proof}
We may assume that $\L_1$ and $\L_2$ are disjoint. Adopt the notations from (\ref{partition1}-\ref{partition3}). Now,
\eqn{
\mat{C&d\\d^T&e} :=  \mat{A & b\\b^T & c}^{-1} = \mat{A^{-1} & 0 \\ 0 & 0} +  \mat{\frac{A^{-1}bb^TA^{-1}}{c-b^TA^{-1}b} & \frac{-A^{-1}b}{c-b^TA^{-1}b} \\ \frac{-b^TA^{-1}}{c-b^TA^{-1}b} & \frac{1}{c-b^TA^{-1}b}}.
} 
Divide vector $d$ by diagonal number $e$ yields
\eqn{
\frac{1}{e}\cdot d = \left. \frac{-b^TA^{-1}}{c-b^TA^{-1}b} \middle/ \frac{1}{c-b^TA^{-1}b} \right.  
= -b^TA^{-1}.
} 
As we have proved in Lemma 4, $-b^TA^{-1}\geq-\t{b}^T\t{A}^{-1}\geq0$, i.e., 
\eqn{
\frac{(L_{(\V-\L_1)}^{-1})_{ij}}{(L_{(\V-\L_1)}^{-1})_{jj}}
\geq 
\frac{(L_{(\V-\L_1\cup\L_2)}^{-1})_{ij}}{(L_{(\V-\L_1\cup\L_2)}^{-1})_{jj}}\geq0 \quad \forall v_i,v_j\not\in\L_1\cup\L_2. \label{eqn:suppress1}
}
Similarly, 
 $\displaystyle
\frac{(L_{(\V-\L_1)}^{-1})_{ij}}{(L_{(\V-\L_1)}^{-1})_{ii}}
\geq 
\frac{(L_{(\V-\L_1\cup\L_2)}^{-1})_{ij}}{(L_{(\V-\L_1\cup\L_2)}^{-1})_{ii}}\geq0.$
It suffices to multiply both sides of the above. 
\end{proof}

\section{Extension to Active Survey \arxivVer{and Tricks to Improve Efficiency}} \label{chpt:survey}

In an active survey problem[7], our goal is to actively query points to ultimately predict the proportion of a given class. Embedded in GRF models, it changes the loss function to the Mean Squared Error (MSE)  $L_s(\v{y}_\T,\v{f}_\T) = (\sum_{v_{t_i}\in\T}(y_{t_i} - f_{t_i}))^2$ and the risk to $R_s(\L) = \v{1}^TTL_{(\V-\L)}^{-1}T^T\v{1}$. 

For the subset selection problem with this new objective function, all results in \S~3 hold and the proofs are similar. Besides, we developed an algorithm with $\mathcal{O}(N^{2.36}+kN^2)$ runtime complexity for $k$ queries ($\mathcal{O}{(kN^{3.36})}$ if implemented natively) similar to [2]%
\switchVer{, detailed in our longer version [9].}{.
\subsection{The Active Surveying Problem and The Proofs}
Similar to (\ref{eqn:problem}), define the subset selection (active surveying) problem as 
\eqn{
\textstyle\arg  \min_{\mathcal{L}}  &\textstyle \quad R(\L)=R_s(\L)=\v{1}^TTL_{(\V-\L)}^{-1}T^T\v{1}   \nonumber\\
s.t. &\textstyle \quad \sum_{v\in \mathcal{L}}  c_v \leq C. \label{eqn:problem2}
}
We also assume (\ref{ppt1}-\ref{ppt4}) and $R_\Delta(\L):=R(\emptyset)-R(\L)$. To prove Theorem 1 via Lemma 2 and 3, the only adjustment is with (\ref{submod:c1}), 
\eqn{
&R_{\Delta}(\L_1\cup\{v\}) - R_{\Delta}(\L_1)
=
 R_s(\L_1) - R_s(\L_1\cup\{v\})  
=  \v{1}^TT\left(\mat{A &b\\b^T & c}^{-1} - \mat{A^{-1} & 0\\0&0}\right)T\v{1} \nonumber\\
=&  \v{1}^TT\left(\mat{-A^{-1}b\\1}\frac{1}{c-b^TA^{-1}b}\mat{-b^TA^{-1} , 1}\right)T\v{1}
=  \frac{1}{c-b^TA^{-1}b}\left(\v{1}^TT\mat{-A^{-1}b\\1}\right)^2 .\label{submod:s1}
}
Still, because $-b\geq-\t{b}\geq0$, $A^{-1}\geq\mat{\t{A}^{-1} & 0 \\0&0}\geq0$, and $T\geq\mat{\t{T}, 0}\geq0$, the above is larger than its counterpart in $R_\Delta(\L_1\cup\L_2\cup\{v\})-R_\Delta(\L_1\cup\L_2)$. \qed

\subsection{Tricks to Improve Efficiency: With Precomputed Covariance}

In Algo 1, the most time-consuming step is to compute $R(\L\cup\{v'\})$ for every possible $v'\in\mathcal{P}$, which in general involves taking the inverse of $L_{(\V-\L\cup\{v'\})}$. Zhu et. al. [5] presented a fast way to do this. Actually it can get even faster in the following way, assuming 
 $L_{(\V-\L\cup\{v'\})}^{-1} = \Sigma'=A^{-1}$, $L_{(\V-\L)}^{-1} = \Sigma =\mat{A&b\\b^T&c}^{-1}=\mat{C&d\\d^T&e}$, and $\Sigma_{*v'}$ to be the last column of $\Sigma$ ,
\eqn{
& \mat{C&d\\d^T&e} = \mat{A^{-1} & 0 \\ 0 & 0} +  \mat{\frac{A^{-1}bb^TA^{-1}}{c-b^TA^{-1}b} & \frac{-A^{-1}b}{c-b^TA^{-1}b} \\ \frac{-b^TA^{-1}}{c-b^TA^{-1}b} & \frac{1}{c-b^TA^{-1}b}}\\
\Rightarrow \quad & \mat{A^{-1}&0\\0&0} = \mat{C&d\\d^T&e}  - \frac{1}{e}\cdot \mat{d\\e}\mat{d^T,e}\\
\Rightarrow\quad & \mat{\Sigma' &0\\0&0 } = \Sigma - {\frac{1}{\Sigma_{v'v'}}\cdot \Sigma_{*v'}\Sigma_{v'*}}.\label{eqn:efficiencywithcov}
} 
In Algo 2, only linear time is needed to evaluate the marginal gain of a candidate because 
\eqn{R_c(\L\cup\{v'\})&=tr(\Sigma') = tr(\Sigma)-tr({\frac{1}{\Sigma_{v'v'}}\cdot \Sigma_{*v'}\Sigma_{v'*}}) = const - \frac{\Sigma_{v'*}\Sigma_{*v'}}{\Sigma_{v'v'}}\nonumber\\ 
R_s(\L\cup\{v'\})&=\v{1}^T\Sigma'\v{1}=\v{1}^T\Sigma\v{1}-\v{1}^T{\frac{1}{\Sigma_{v'v'}}\cdot \Sigma_{*v'}\Sigma_{v'*}}\v{1}=const-\frac{(\v{1}^T\Sigma_{*v'})^2}{\Sigma_{v'v'}} \nonumber
}

\begin{algorithm}[H]
 \caption{Fast progressive $R(\L\cup\{v'\})$ evaluation with precomputed covariance.}\label{algo:fastProgressive}
 \normalsize
 \KwIn{Labeled set $\L$, current $R(\L)$, $\Sigma$ (covariance of $\v{y}_\U$ conditioned on $\v{y}_\L$), queryable pool $\mathcal{P}$, and test set $\T\subset\U$ if applicable (otherwise $\T\leftarrow I$ of size $|\U|$).}
 \KwOut{$R(\L\cup\{v_{p_1}\}), ..., R(\L\cup\{v_{p_{|\mathcal{P}|}}\}$.}
 $T \leftarrow (\delta(v_{t_i}, v_{u_j}))_{i=1,j=1}^{|\T|,|\U|}$ , or $T \leftarrow I$ of size $|\U|$ if $\T$ not specified.\\
 \For{$v_{p_i}\in\mathcal{P}-\L$}{
 $v'\leftarrow j$ if $\displaystyle v_{u_j}= v_{p_i} $ .\\
  $R(\L\cup\{v_{p_i}\}) \leftarrow R(\L) - \displaystyle \frac{(\Sigma_{v'*}T^T)(T\Sigma_{*v'})}{\Sigma_{v'v'}}$ if classification or $R(\L) -\displaystyle \frac{(\v{1}^TT\Sigma_{*v'})^2}{\Sigma_{v'v'}}$ if survey.
  }
\end{algorithm}

\subsection{Tricks to Improve Efficiency: Singular Laplacian}

However, we still have one question unsolved---how to compute the first $L^{-1}$ when $L$ for a connected graph is singular? 

The algorithm for classification problem $\arg\min_{v_0}tr(L_{(\V-\{v_0\})}^{-1})$ has been optimized in [2]. We can follow a similar method to compute $\arg\min_{v_0}\v{1}^TL_{(\V-\{v_0\})}^{-1}\v{1}$ and also the criterion with specified test sets. Essentially, we want to avoid numerical inverse of large matrices as much as possible. In fact, both the algorithm in [2] and the following require only one eigen-decomposition of $L$ , which has the same order of complexity as matrix inversion.

\begin{definition}[First Query in Survey Problem]
Suppose $L$ satisfies (\ref{ppt1}-\ref{ppt3}), i.e. every property including connectivity but singularity. Also suppose $L$ has eigen-decomposition $L=Q\Lambda Q^T$, where $ \Lambda =  \diag{\lambda_1,  \lambda_2, ..., \lambda_N}$ with $\lambda_1=0, \lambda_k>0, \forall k\neq1$ and $Q$ is the orthogonal matrix whose every column is the regularized eigenvector corresponding to the eigenvalue in $\Lambda$. Denote the \emph{row vector representation}\footnote{ Notice this $r_i$ representation is the only row vector representation in this paper.}  of $Q$ as $Q=\mat{r_1\\r_2\\\cdots\\r_N}$  and its miss-ith-row form  $Q_{-i,*}=\mat{r_1\\\cdots\\r_{i-1}\\r_{i+1}\\\cdots\\r_N}.$ 
The first query in survey problem asks to optimize \eqn{
\arg\min_{i} R_s(\{v_i\}) = \v{1}^T L_{\V-\{v_i\}}^{-1} \v{1} 
= \v{1}^T \cdot \left( Q_{-i,*} \Lambda Q_{-i,*}^T\right)^{-1} \cdot\v{1}
}
\end{definition}
\begin{proof}[Solution (First Query in Survey Problem)]$\quad$\\
For any fixed $i$, denote (n-1)-by-n $\tilde{Q} = Q_{-i,*}$. Thus $\t{Q}\t{Q}^T=I_{N-1}$, $\t{Q}^T\t{Q} = I_N-r_i^Tr_i$, and $R_s(\{v_i\}) = \v{1}^T(\t{Q}\Lambda\t{Q}^T)^{-1}\v{1}$. Also denote $\Lambda =\mat{\lambda_1 & 0\\ 0 & \t{\LL}}= \mat{0 & 0\\ 0 & \t{\LL}}$, where $\t{\LL}$ is (n-1)-by-(n-1) nonsingular diagonal matrix and $\tilde{L} = L_{(\V-\{v_i\})} = \t{Q}\Lambda\t{Q}^T$. By matrix inversion theorem, 
\eqn{
\t{L}^{-1} &= \mat{- \t{Q}\t{Q}^T + \t{Q}(I_N+\Lambda)\t{Q}^T}^{-1} \\
&=(-\t{Q}\t{Q}^T)^{-1} - (\t{Q}\t{Q}^T)^{-1}\t{Q}\Big[(I_N+\LL)^{-1} + \t{Q}^T(-\t{Q}\t{Q}^T)^{-1}\t{Q}\Big]^{-1}\t{Q}^T(\t{Q}\t{Q}^T)^{-1}\\
& = -I_{N-1} - \t{Q}\Big[(I_N+\LL)^{-1} - \t{Q}^T\t{Q}\Big]^{-1}\t{Q}^T\\
%
%
& = -I_{N-1} - \t{Q}\Big[\Big((I_N+\LL)^{-1} - I_N\Big) +  r_i^Tr_i\Big]^{-1}\t{Q}^T
%
}
Since $L$ is a connected graph Laplacian,  the normalized eigenvector for $\lambda_1=0$ is $\mat{\frac{1}{\sqrt{N}}...\frac{1}{\sqrt{N}}}^T$. Therefore, we can denote $r_i = \mat{\frac{1}{\sqrt{N}}, \ai^T}$, where $\ai$ is $(N-1)$-dimensional. Apply matrix inversion theorem again, 
\eqn{
&\left[\Big((I_N+\LL)^{-1} - I_N\Big) +  r_i^Tr_i\right]^{-1} 
%
= \mat{\frac{1}{N} & \frac{\ai^T}{\sqrt{N}}\\\frac{\ai}{\sqrt{N}} & \underbrace{\diag{\frac{-\lambda_k}{1+\lambda_k}}_{k=2}^N}_{\tilde{M}} + \ai\ai^T}^{-1}\\
=&\mat{\frac{1}{N} & \frac{\ai^T}{\sqrt{N}}\\\frac{\ai}{\sqrt{N}} & \underbrace{\tilde{M} + \ai\ai^T}_{\tilde{B}}}^{-1}
=\mat{\frac{1}{m}  & -\frac{1}{m\sqrt{N}}\ai^T\t{B}^{-1}  \\ -\frac{1}{m\sqrt{N}}\t{B}^{-1}\ai & \t{B}^{-1} + \frac{1}{m}\t{B}^{-1}\frac{\ai\ai^T}{N}\t{B}^{-1}},
} 
where $
\t{B}^{-1}  
=\t{M}^{-1} - \t{M}^{-1}\ai\cdot\frac{1}{1 + \ai^T\t{M}^{-1}\ai}\cdot\ai^T\t{M}^{-1}$
and $m = \frac{1}{N} - \frac{1}{N}\ai^T\t{B}^{-1}\ai.$

Assign $a_i = \ai^T\t{M}^{-1}\ai$ and we have 

\eqn{
\ai^T\t{B}^{-1}\ai &= \ai^T\t{M}^{-1}\ai -\frac{(\ai^T\t{M}^{-1}\ai)^2}{1 + \ai^T\t{M}^{-1}\ai} =a_i - \frac{a_i^2}{1+a_i} = \frac{a_i}{1+a_i} \\
\frac{1}{m} &= \Big(\frac{1}{N}-\frac{1}{N}( \frac{a_i}{1+a_i})\Big)^{-1}= N(1+a_i).
}

Finally, because the first column of the orthogonal $Q$ is $\frac{1}{\sqrt{N}}\v{1}$, we have $\v{1}^TQ=\mat{\sqrt{N},\v{0}^T}$ and 

\eqn{
& R_s(\{v_i\}) =\v{1}^T\cdot\t{L}^{-1}\cdot\v{1}\\
=&-(N-1) -(\v{1}^TQ-r_i)\mat{\frac{1}{m}  & -\frac{1}{m\sqrt{N}}\ai^T\t{B}^{-1}  \\ -\frac{1}{m\sqrt{N}}\t{B}^{-1}\ai & \t{B}^{-1} + \frac{1}{m}\t{B}^{-1}\frac{\ai\ai^T}{N}\t{B}^{-1}}(\v{1}^TQ-r_i)^T \\
=&-(N-1) -\mat{\frac{N-1}{\sqrt{N}}, -\ai}\mat{\frac{1}{m}  & -\frac{1}{m\sqrt{N}}\ai^T\t{B}^{-1}  \\ -\frac{1}{m\sqrt{N}}\t{B}^{-1}\ai & \t{B}^{-1} + \frac{1}{m}\t{B}^{-1}\frac{\ai\ai^T}{N}\t{B}^{-1}}\mat{\frac{N-1}{\sqrt{N}} \\ -\ai^T} \\
=&-(N-1) - \Big[ (N-1)^2(1+a_i) + 2 (N-1)a_i + \frac{a_i}{1+a_i} + \frac{a_i^2}{1+a_i} \Big] \\
=&-N (N-1) - N^2a_i ,
}

where $a_i = \mat{q_{i,2},...,q_{i,N}}\diag{\frac{1+\lambda_2}{-\lambda_2}, ... ,\frac{1+\lambda_N}{-\lambda_N}} \mat{q_{i,2},...,q_{i,N}}^T$.

When a test set $\T$ is specified, since $v_i\not\in \T$, 
\eqn{R_s(\{v_i\})&=\v{1}^TT\t{L}^{-1}T^T\v{1}=-(|\T|-1)-\v{1}^TT\t{Q}\left[\left(I_N+\LL\right)^{-1}+r_i^Tr_i\right]^{-1}\t{Q}^TT^T\v{1}\\&=-(|\T|-1)-\v{1}^TTQ\left[\left(I_N+\LL\right)^{-1}+r_i^Tr_i\right]^{-1}Q^TT^T\v{1}.} 
A similar algorithm can be derived, though the runtime complexity may has a factor $|\T|$. 
\end{proof}

\begin{algorithm}[H]
 \caption{Fast first-step $R_s(\{v\})$ evaluation with singular Laplacian.}\label{algo:fastFirstStep}
 \normalsize
 \KwIn{Singular connected graph Laplacian $L$.}
 \KwOut{$R_s(\{v_i\}) = \v{1}^TL_{(\V-\{v_i\})}^{-1}\v{1},\quad i=1,\dotsc,N$.}
 Perform eigen-decomposition $L=Q\LL Q^T$, where $\LL=\diag{\lambda_1,...,\lambda_N}$ in ascending order.\\
 Denote  $M^{-1} \leftarrow \diag{0,\frac{1+\lambda_2}{-\lambda_2}, ... ,\frac{1+\lambda_N}{-\lambda_N}}$ and $Q=\mat{r_1\\\cdots\\r_N}$.\\
 \For{$i=1,...,N$}{
 $a_i\leftarrow r_i M^{-1} r_i^T$.\\
 $R_s(\{v_i\}) = -N(N-1)-N^2a_i.$
  }
\end{algorithm}

}

\section{Experiment}

We performed various active learning methods on the DBLP coauthorship graph dataset\footnote{\scriptsize\url{http://www.informatik.uni-trier.de/~ley/db/}} on four areas: machine learning, data mining, information retrieval and database. Edge weights are the number of papers coauthored. We took its largest connected component, which contains 1711 nodes and 0.3\% of all possible edges. 
We used the V-optimality criterion (\S~2), mutual information gain $(\max_\L \mathcal{MI}(\L;\V-\L))$ [1], and random selection. For fair comparison, every method was assigned the same\draftfootnote{Actually, with 4 random seeds, each in a class, our V-optimality curve becomes worse than random after about 60 queries. So I suppose we should stay with the better old result?} random seed to start and the curves are the mean and the standard error of the mean after 120 repetitions (Figure~1). The V-optimality criterion performs better than others.
\begin{figure}[htbp]
\begin{center}
\includegraphics[width=.7\linewidth]{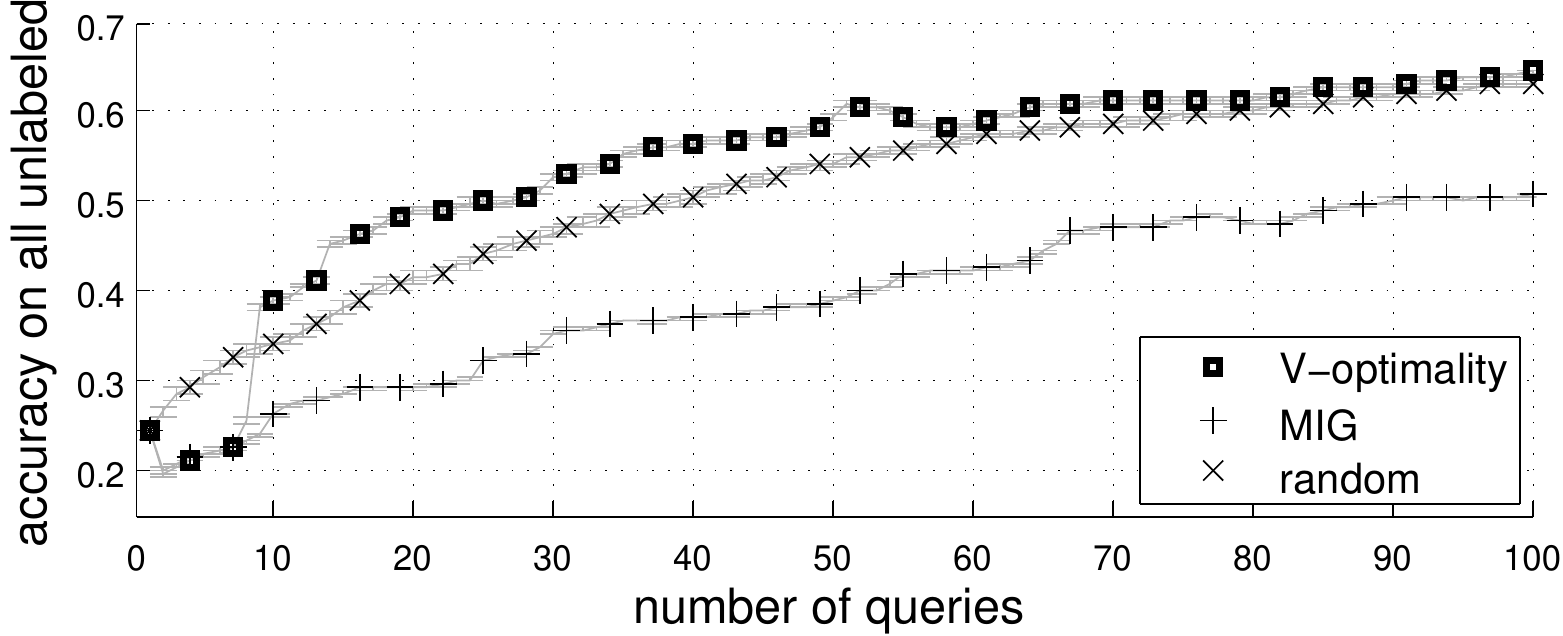}
\vspace{-1em} 
\caption{Batch active learning to classify the unlabeled authors
on DBLP coauthorship graph.}
\label{default}
\end{center}
\end{figure}
\vspace{-2em}
\section{Conclusion}
In this paper, we introduced the GRF model (\ref{eqn:theGRFmodel}) and the Gaussian harmonic prediction (\ref{eqn:GaussianHarmonicPrediction}). The batch active learning with V-optimality criterion, whose risk function is (\ref{eqn:Risk}) can be formulated as the subset selection problem (\ref{eqn:problem}). Our major contribution is to prove the submodularity conditions (\ref{eqn:submod1}-\ref{eqn:submod3}) and an $(1-1/e)$ optimality bound (\ref{eqn:1-1/e}) for a greedy selection algorithm (Algo 1) when the graph Laplacian is nonsingular (\ref{ppt1}-\ref{ppt4}), via either extracting a subgraph from the original connected graph or regularizing the GRF model. Furthermore, the fact that all GRFs meet the AofS condition (Theorem 2) may shed light on this otherwise obscure condition. 

In \S~4, we also proposed an active survey problem and its related risk $R_s(\L)$. We can show that this batch active survey problem also meet the submodularity conditions and its greedy subset selection algorithm achieves a similar $(1-1/e)$ optimality bound.

\subsubsection*{References}

\small{
[1] Andreas Krause, Aarti Singh, Carlos Guestrin. Near-optimal Sensor Placements in Gaussian Processes: Theory, Efficient Algorithms and Empirical Studies. Journal of Machine Learning Research (JMLR) 2008.

[2] Ming Ji and Jiawei Han. A Variance Minimization Criterion to Active Learning on Graphs.
AISTAT 2012.

[3] Abhimanyu Das and David Kempe. Submodular meets Spectral: Greedy Algorithms for Subset Selection, Sparse Approximation and Dictionary Selection. 
ICML 2011

[4] Abhimanyu Das and David Kempe. Algorithms for Subset Selection in Linear Regression. ACM Symposium on Theory of Computing, STOC 2008.

[5] Xiaojun Zhu, John Lafferty, and Zoubin Ghahramani. Combining active learning and semi-supervised learning using gaussian Þelds and harmonic functions. In the workshop on The Continuum from Labeled to Unlabeled Data in Machine Learning and Data
Mining, ICML 2003.

[6] Matthew Streeter\&Daniel Golovin.An Online Algorithm for Maximizing Submodular Functions.NIPS2008.

[7] Roman Garnett et. al. Bayesian Optimal Active Search and Surveying. ICML 2012.

[8] Burr Settles. Active Learning Literature Survey. Computer Sciences Technical Report 1648, University of Wisconsin-Madison. 2009.

\wsVer{[9] Yifei Ma et. al. Submodularity in Batch Active Learning and Survey Problems on Gaussian Random Fields. arXiv:xxxx.xxxx. 2012. \draftfootnote{the id}}
}

\end{document}